\documentclass{article}

\usepackage{arxiv}

\usepackage[utf8]{inputenc} 
\usepackage[T1]{fontenc}    
\usepackage{hyperref}       
\usepackage{url}            
\usepackage{booktabs}       
\usepackage{amsfonts}       
\usepackage{nicefrac}       
\usepackage{microtype}      
\usepackage{lipsum}

\usepackage{url}
\usepackage{color}
\usepackage{algorithmic,algorithm}
\usepackage{graphicx}
\usepackage{subfigure}

\usepackage{amsmath,amssymb,bbm,amsthm}
\newtheorem{theorem}{Theorem}

\newtheorem{lemma}[theorem]{Lemma}
\newtheorem{definition}{Definition}

\author{
   Shijie Xu\qquad Jiayan Fang\qquad Xiang-Yang Li\thanks{Corresponding author.}\\
  School of Computer Science and Technology\\
  University of Science and Technology of China\\
  230027, Hefei, Anhui, P.R.China \\
  \texttt{\{xushijie,jyfang\}@mail.ustc.edu.cn, xiangyangli@ustc.edu.cn} \\
}

\title{Weighted Laplacian and Its Theoretical Applications}

\begin{document}
%
\maketitle
\begin{abstract}
\begin{quote}
In this paper, we develop a novel weighted Laplacian method, which is partially inspired by the theory of graph Laplacian, to study recent popular graph problems, such as multilevel graph partitioning and balanced minimum cut problem, in a more convenient manner. Since the weighted Laplacian strategy inherits the virtues of spectral methods, graph algorithms designed using weighted Laplacian will necessarily possess more robust theoretical guarantees for algorithmic performances, comparing with those existing algorithms that are heuristically proposed. In order to illustrate its powerful utility both in theory and in practice, we also present two effective applications of our weighted Laplacian method to multilevel graph partitioning and balanced minimum cut problem, respectively. By means of variational methods and theory of partial differential equations (PDEs), 
we have established the equivalence relations among the weighted cut problem, balanced minimum cut problem and the initial clustering problem that arises in the middle stage of graph partitioning algorithms under a multilevel structure. These equivalence relations can indeed provide solid theoretical support for algorithms based on our proposed weighted Laplacian strategy. Moreover, from the perspective of the application to the balanced minimum cut problem, weighted Laplacian can make it possible for research of numerical solutions of PDEs to be a powerful tool for the algorithmic study of graph problems. Experimental results also indicate that the algorithm embedded with our strategy indeed outperforms other existing graph algorithms, especially in terms of accuracy, thus verifying the efficacy of the proposed weighted Laplacian.
\end{quote}
\end{abstract}

\section{Introduction}
In graph theory, the traditional weighted graph consists only of weights on its edges. Although weighting edges has several practical uses such as scientific simulation, social networks and integrated circuit design, weighting vertices can also serve as an important role for many purposes, e.g. the Hosoya polynomial and Wiener index \cite{wiener1947structural} of vertex-weighted graph has been well studied since its extensive applications in chemical graph theory. The \textit{doubly-weighted} graph --- a graph in which both edges and vertices are weighted --- however, is very different from the above two versions of graph and proves to have vital potentials in solving some detailed issues which have not yet studied clearly in some practical problems.

Partial differential equations (PDEs) in graph analysis has been well studied over decades and has gradually developed into a powerful framework to help us understand various phenomena that might occur on graphs. One of particular success lies in the field of spectral graph theory, in which the Laplacian operator has been always playing a crucial role in dealing with some popular problems about graphs.
Moreover, the Laplacian operator, acting as a core operator for many second-order PDEs such as Laplace's equations, Poisson's equations, and evolution equations, has already given birth to the emergence of various mighty techniques in the area of graph analysis. In some applications to graph processing, these equations have been discretized by resorting to several approximate methods including finite elements, finite volumes and Monte Carlo simulation \cite{sun2009concise}.
Such kind of Laplacian operators capable of coping with graph problems are uniformly termed as graph Laplacian.
Some notions like $p$-Laplacian and $\infty$-Laplacian have been used to describe some important process in physics, biology, or economy \cite{drabek2007p,oberman2013finite}. Additionally, the nonlocal $p$-Laplacian has gained growing interest in various areas like mathematical biology, peridynamics, and image processing \cite{andreu2010nonlocal,elmoataz2015p}. 

As a successful application to graph Laplacian, 
spectral graph theory has evolved to be a sort of powerful techniques in the topic of clustering and has been developed systematically over decades \cite{chung1997spectral}. It was deemed to build originally from the minimum cut problem on the graph, and another motivation might come from the graph energy problem \cite{li2012graph}, in which the cut of the graph partition can be regarded as the energy of partition \cite{kolmogorov2002energy}. 
Since most applications to graph analysis are based on discretization methods, spectral methods based on graph Laplacian have been sufficiently exploited for minimizing some particular graph partitioning objective functions, such as the cheeger cut 
\cite{li1979lower,chung1997spectral}, the ratio cut \cite{hagen1992new} and the normalized cut \cite{shi2000normalized}, and can always return the global optimal solution to relaxation of these objective functions. These objective functions also fulfill to establish significant connections among various practical problems due to their strong mathematical background.
For example, there exists several important equivalence relations between the total variation problem and the ratio cut problem \cite{szlam2010total}, the weighted kernel $k$-means problem and the normalized cut problem \cite{dhillon2004unified}, and the normalized cut problem and the modularity maximization problem based on spectral methods in community detection \cite{newman2013community}, respectively.

Considering that most existing graph Laplacians are built on the form of either edge-only-weighted graphs or vertex-only-weighted graphs \cite{chung1996combinatorial,knisley2014vertex}, techniques developed from these graph Laplacians may suffer from limited applications when confronted with more complex situations. On account of vital potentials of doubly-weighted graphs in coping with some recently important unsolved problems \cite{carron2016hitting,kim2019harmonic}, we propose the novel weighted Laplacian method inspired by existing theory of graph Laplacian to make it possible for the Laplacian operator suitable for doubly-weighted graphs, for the purpose of extending the practical utility of graph Laplacians. In order to further demonstrate advantages and wide potential use of our proposed strategy, we provide two theoretical applications of the weighted Laplacian method respectively to the design of clustering algorithms for the multilevel graph partitioning problem and the balanced minimum cut problem. The main contributions of this paper are listed as follows:
\begin{enumerate}
    \item We create the notion of weighted Laplacian in the context of doubly-weighted graphs by resorting to the theory of partial differential equations on graphs, and further propose the weighted Laplacian method by the inspiration of existing spectral methods.
    \item For the purpose of demonstrating the theory value of our idea, we give the first application of the weighted Laplacian method to the balanced minimum cut problem. We provide rigorous proof from the perspective of PDEs about the equivalence between the weighted cut problem and the balanced minimum cut problem both in their relaxed versions, thus revealing the possibility that the mature research of numerical solutions of PDEs will be of great assistance in recent studies of graph problems.
    \item As a second application, we also show how to embed our weighted Laplacian method into the design of algorithms for multilevel graph partitioning problems. In order to illustrate the practical power of the proposed idea, we prove that the weighted cut problem is essentially equivalent to the initial clustering problem that appears in the second phase of those graph partitioning algorithms with multilevel structure. Furthermore, we post the description of the multilevel graph partitioning algorithm based on the proposed strategy.
\end{enumerate}

The rest of the paper is organized as follows. Section \ref{sec:related} provides a brief overview on some related work including the balanced minimum cut problem and multilevel structure of the graph partitioning problem. In section \ref{sec:prelim}, we give the definition of weighted Laplacian and its important properties, then the weighted Laplacian method is proposed later. Section \ref{sec:thy} will mainly dwell on providing two theoretical applications of our weighted Laplacian method and the detailed practical algorithm of one application will be posted. Section \ref{sec:expr} reports the performance of our algorithm. Finally, we conclude in section \ref{sec:conc}.
\section{Related Work}\label{sec:related}
Before diving into the details of our work, we first take a brief overview on some excellent research that are closely related to our method.
\subsection{Graph Laplacian and Balanced Minimum Cut Problem} 
The theoretical foundations of spectral graph theory stem originally from the work of \cite{hall1970r}, and were further developed in decades. Graph Laplacians lie in the heart of major spectral methods. There are several different Laplacians based on edge-only-weighted graphs given in related literatures, e.g. the unnormalized Laplacian
\begin{equation*}
    L:=D-W,
\end{equation*}
the normalized Laplacian
\begin{equation*}
    L_{\mathrm{N}}:=I-D^{-1/2}WD^{-1/2},
\end{equation*}
and the random walk Laplacian
\begin{equation*}
    L_{\mathrm{rw}}:=I-D^{-1}W,
\end{equation*}
where $W=\{W_{ij}\}$ and $D=\mathrm{diag}\{d_i\}$ are the weight matrix and the degree matrix of the graph respectively. Moreover, the vertex-only-weighted graph Laplacians have also been continually studied \cite{friedman2004calculus,knisley2014vertex,shi2016weighted}. All of the practical Laplacians mentioned above, however, consider either the edge-only-weighted or vertex-only-weighted graphs. Considering that the minimum cut problem becomes increasingly important in the context of doubly-weighted graphs, to which most existing algorithms are unfortunately difficult to be extended \cite{liu2014weighted}, it is worthwhile to develop a new kind of graph Laplacian suitable for doubly-weighted graphs.

In recent years, the balanced minimum cut problem also constantly serves as an important role in various practical situations. There exists a number of different ways to define a series of balance conditions which act as a class of constraints in the balanced minimum cut optimization problem. Based on some particular definitions of the balance conditions, a balanced partition can be produced as a solution of the balanced minimum cut problem via some certain algorithm. See the work of \cite{chen2017self,liu2017balanced} for further references. The notion of balanced minimum cut is actually ubiquitous in many graph problems for the reason that some graph partitioning objective functions in spectral graph theory can be essentially regarded as a sort of ``balanced'' minimum cut, in whose corresponding optimization problem, each cut term expressed in the ratio form respectively has a predefined balance condition as its denominator, such as the cardinality of each partition (ratio cut) or the volume of each partition (normalized cut). Besides, the definition of the perfectly balanced minimum cut can be found in \cite{andreev2006balanced}.

\subsection{Multilevel Graph Partitioning}
Under the background of multilevel graph partitioning (also called V-cycle), a plenty of heuristics with different nature have recently been successively developed due to its practical significance 
\cite{ubaru2019sampling}
in order to design efficient approximation algorithms for graph partitioning problems with reasonable computational time. However, not until the general-purpose multilevel methods were put forward, had not the field of graph partitioning undertaken a truly breakthrough in aspects of both efficiency and partition quality \cite{safro2015advanced}. A specific multilevel graph partitioning algorithm consists of three phases: coarsening --- where the problem instance is gradually mapped to a smaller one to reduce the original complexity, initial clustering --- where the coarsening graph is partitioned by some specific clustering algorithm, and refining --- where the partition for original graph is inversely refined from coarsened partitioning, as shown in figure \ref{fig:multi}.

\begin{figure}[t]
    \centering
    \includegraphics[width=0.5\textwidth]{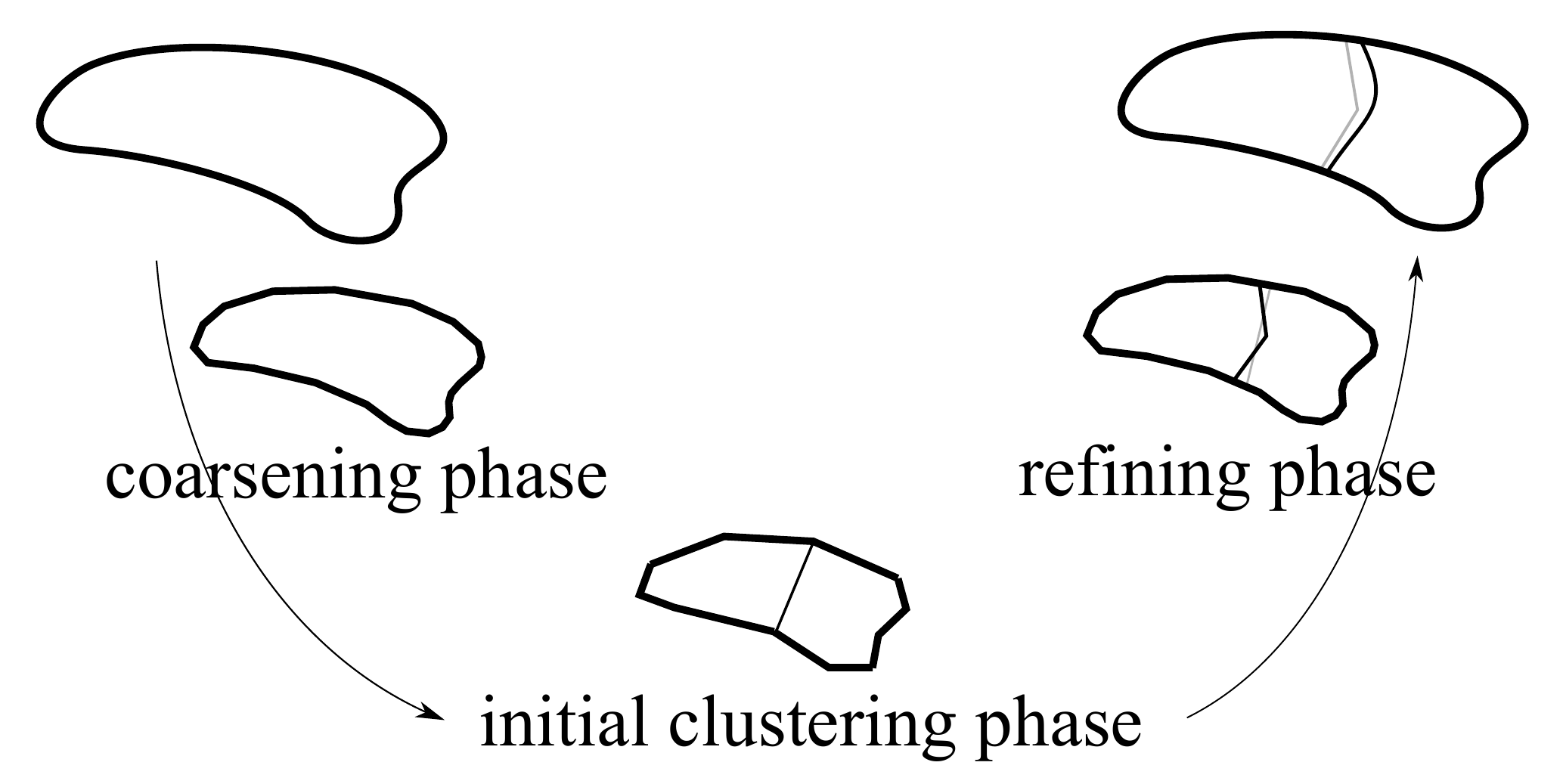}
    \caption{Three phases in multilevel graph partitioning: the large graph is coarsened during the coarsening phase; an initial clustering algorithm is running on the coarsened graph to obtain a coarsened partition; finer partition is produced during the refining phase.}
    \label{fig:multi}
\end{figure}

\subsubsection{Coarsening phase}
We denote $G_i$ as the $i$-th graph and $V_i$ as the corresponding vertex set of the $i$-th graph.  Starting with one original graph $G_0$, a specific coarsening algorithm repeatedly transforms the graph into smaller and smaller graphs $G_1,G_2,\ldots,G_m$ such that $|V_0|>|V_1|>\ldots>|V_m|$. In order to obtain a coarser graph from $G_i$ to $G_{i+1}$, all nodes in $G_i$ will be partitioned by some rules into several groups, each of which can be regarded as a supernode in $G_{i+1}$. Several criteria for grouping the nodes are proposed in \cite{dhillon2004unified,holtgrewe2010engineering}. In our settings, when combining a set of nodes into a single supernode, the edge weight of the supernode is taken to be the sum of the edge weights of the original nodes that comprise this supernode. Similarly, the degree of a supernode is taken to be the sum of the degrees of all original nodes that are contained in it. 
See more details in \cite{safro2015advanced}.

\subsubsection{Initial clustering phase}
Several initial clustering algorithms have been proposed and further developed, e.g. the region-growing algorithm \cite{karypis1998fast,sanders2013think,predari2016k}, 
the recursive bisection algorithm \cite{barnard1994fast,aykanat2008multi}, spectral clustering algorithm and weighted kernel algorithm \cite{dhillon2004unified,dhillon2007weighted}. 
However, all of the above initial clustering algorithms are heuristic, thus failing to provide qualitative guarantee for partitioning results of coarsened graphs. Then, a better initial method is immediately desired, considering that the coarsened graph partition produced by the initial clustering phase, as a intermediate result in the whole multilevel algorithmic framework, can certainly exert significant impacts on the subsequent refining phase in aspects of both time complexity and partition accuracy.

\subsubsection{Refining phase}

As the final step of the multilevel framework, given a partition of graph $G_i$, the refining phase forms a finer partition of the graph $G_{i-1}$ where $G_i$ is a coarsened version of $G_{i-1}$.
If there is a partition of $G_i$, which naturally yields a partition of $G_{i-1}$ by projecting, then we run the refining algorithm on the $G_{i-1}$ to get a finer partition. 
The Kernighan-Lin objective \cite{lin1973effective} is used to search the local minima by swapping points between different partitions. 
\cite{sanders2013think} also designed a local search algorithm to refine the coarsened partition based on negative cycle detection, where a negative cycle corresponds to a set of node movements that will not only decrease the overall cut but also maintain the balance of a partition as well.
More popular refining algorithms can be found in \cite{safro2015advanced}.

\section{Weighted Laplacian}\label{sec:prelim}
In this section, we propose the notion of weighted Laplacian, which is an extension of graph Laplacian and has similar properties on the doubly-weighted graph as traditional graph Laplacian. 
We first give several important definitions and lemmas, then propose our weighted Laplacian method.

\subsection{Some Definitions and Lemmas}
\begin{definition}[Doubly-weighted graph]\label{def:dou}
Let $G=(V,M,W)$ be a connected undirected graph with doubly weight where $V=\{1,2,\ldots,n\}$ is the vertex set of $G$, and vertex-weight matrix $M=\mathrm{diag}\{m_1,m_2,\ldots,m_n\}$ s.t. $m_i>0$ and edge-weight matrix $W=\{W_{ij}\}$. Let $D=\mathrm{diag}\{d_1,d_2,\ldots,d_n\}$ be the degree matrix of the graph where $d_i=\sum_j W_{ij}$. We say $G=(V,M,W)$ is a doubly-weighted graph.

\end{definition}

\begin{definition}[Weighted cut]
Suppose $A,B$ are two disjoint subsets of $V$, the cut of $A$ and $B$ on the graph $G$ is $\mathrm{Cut}(A,B):=\sum_{i\in A,j\in B}W_{ij}$. Let $\pi=\{C_1,C_2,\ldots,C_k\}$ be a $k$-partition of the graph $G$, i.e. $V=C_1\cup C_2\cup\ldots\cup C_k$ and $C_i\cap C_j=\emptyset$ for any $i\neq j$. 
We define the weighted cut of the partition $\pi$ as follows
\begin{equation*}
    \mathrm{Wcut}(\pi):=\sum_{i=1}^k\frac{\mathrm{Cut}(C_i,\bar{C_i})}{\mathrm{mvol}(C_i)}
\end{equation*}
where $\mathrm{mvol}(C_i)=\sum_{x\in C_i}m_x$. Besides, we recall the definition of the normalized cut as follows,
\begin{equation*}
    \mathrm{Ncut}(\pi):=\sum_{i=1}^k\frac{\mathrm{Cut}(C_i,\bar{C_i})}{\mathrm{vol}(C_i)}
\end{equation*}
where $\mathrm{vol}(C_i)=\sum_{x\in C_i}d_x$.
\end{definition}
The \textit{weighted cut problem} is aiming to minimize the above weighted cut for all partition $\pi$ on the doubly-weighted graph. In order to find the optimal solution of this minimization problem, we employ the theory of partial differential equations on graphs. 
We first give some related definitions and lemmas in the following, then naturally introduce our weighted Laplacian method and show how to apply it in the weighted cut problem.

\begin{definition}[Weighted Laplacian]
Suppose $G=(V,M,W)$ is a doubly-weighted graph. Let $\mathcal{G}$ be the linear space of all functions $f:V\to\mathbb{R}$, we define
the gradient of $f$ as a vector $\nabla f:=((f(y)-f(x))\sqrt{\frac{W_{xy}}{m_x}})_{y\in V}$, and the weighted Laplacian $\Delta$ is an operator in $\mathcal{G}$ defined as $\Delta f:=\sum_{y\in V}(f(x)-f(y))\frac{W_{xy}}{m_x}$. The integral of $f$ is defined as $\int f:=\sum_{x\in V}f(x)m_x$, and the inner product in $\mathcal{G}$ is defined as $\left<f,g\right>=\int fg$ for all $f,g\in\mathcal{G}$.
\end{definition}

Following lemma gives an important property of the weighted Laplacian. 
\begin{lemma}\label{lem:equiv}
     $\Delta$ is equivalent to the weighted Laplacian matrix 
    \begin{equation*}
            L_M=M^{-1/2}(D-W)M^{-1/2}\in\mathbb{R}^{n\times n}.
    \end{equation*}
\end{lemma}
\begin{proof}
    Consider a group of normalized orthogonal bases $\{\delta_i/\sqrt{m_i}\}$ of $\mathcal{G}$ where $\delta_i(x)=1$ if $x=i$ and zero otherwise. We have
    \begin{equation*}
    \begin{split}
        \Delta(i,j)=&\left<\frac{\delta_i}{\sqrt{m_i}},\Delta(\frac{\delta_j}{\sqrt{m_j}})\right>\\
        =&\frac{1}{\sqrt{m_im_j}}\sum_x\sum_y\delta_i(x)(\delta_j(x)-\delta_j(y))W_{xy}\\
        =&\frac{1}{\sqrt{m_im_j}}\sum_y(\delta_j(i)-\delta_j(y))W_{iy}\\
        =&\frac{1}{\sqrt{m_im_j}}(\delta_j(i)\sum_yW_{iy}-\sum_y\delta_j(y)W_{iy})\\
        =&\frac{1}{\sqrt{m_im_j}}(\delta_j(i)d_i-W_{ij})\\
        =&\begin{cases}\displaystyle
                 \frac{-W_{ij}}{\sqrt{m_im_j}},       &  i\neq j\\
        \displaystyle\frac{d_i-W_{ii}}{\sqrt{m_im_i}},      & i=j
        \end{cases}\\
        =&L_M(i,j),
    \end{split}
    \end{equation*}
    therefore $\Delta$ is equivalent to $L_M$. Notice that when $M=I$ or $M=D$ weighted Laplacian becomes unnormalized Laplacian or normalized Laplacian.
\end{proof}

\begin{lemma}\label{lem:ray}
Suppose $\mathbbm{1}_{C_i}\in\mathcal{G}$ be the indicating function of $C_i$, i.e. $\mathbbm{1}_{C_i}(x)=1$ if $x\in C_i$, and zero otherwise. We have
\begin{equation*}
\begin{split}
        \mathrm{Cut}(C_i,\bar{C_i})=\int |\nabla \mathbbm{1}_{C_i}|^2,\quad\mathrm{mvol}(C_i)=\int \mathbbm{1}_{C_i}^2,\\
    \left<f,\Delta f\right>=\int|\nabla f|^2,\quad \left<f,f\right>=\int f^2\quad\mathrm{for~all~}f\in\mathcal{G}.
\end{split}
\end{equation*}
Therefore we have the following equation
\begin{equation}\label{eq:rq}
    \sum_{i=1}^k\frac{\mathrm{Cut}(C_i,\bar{C}_i)}{\mathrm{mvol}(C_i)}=\sum_{i=1}^k\frac{\int|\nabla\mathbbm{1}_{C_i}|^2}{\int\mathbbm{1}_{C_i}^2}=\sum_{i=1}^k\frac{\left<\mathbbm{1}_{C_i},\Delta\mathbbm{1}_{C_i}\right>}{\left<\mathbbm{1}_{C_i},\mathbbm{1}_{C_i}\right>}.
\end{equation}
\end{lemma}
The equations in this lemma are trivial to verify and we omit the proof here to save space.

\subsection{Weighted Laplacian Method}
We are inspired by the spectral methods, which first relaxes the discreteness condition in the optimization problem such as in the normalized cut and the ratio cut to solve them, then re-converts a partition from the real-valued solution. Accordingly, the weighted Laplacian method also first aims to solve the relaxation version of weighted cut problem, then its global optimal solution is produced. We present the condition to achieve its minimum values using the variational approach, based on the Rayleigh quotient presentation as follows.

Rayleigh quotient $\mathfrak{R}$ is a functional on $\mathcal{G}$ defined as following
\begin{equation*}
    \mathfrak{R}(f)=\frac{\left<f,\Delta f\right>}{\left<f,f\right>}\quad\mathrm{for~all~}f\in\mathcal{G},
\end{equation*}
and the weighted cut can be presented as a functional $\mathfrak{L}$ of indicating functions $f_1,f_2,\ldots,f_k$ shown as following
\begin{equation*}
    \mathfrak{L}(f_1,f_2,\ldots,f_k)=\sum_{i=1}^k\mathfrak{R}(f_i).
\end{equation*}
Relaxing the discreteness condition we have the relaxed version of weighted cut problem. In order to minimize functional $\mathfrak{L}$, from Euler-Lagrange equation with several functions we obtain that
\begin{equation*}
    \frac{\partial \mathfrak{L}}{\partial f_i}=\frac{\partial\mathfrak{R}(f_i)}{\partial f_i}=0\quad\mathrm{for~all}~i=1,\ldots,k,
\end{equation*}
therefore we calculate the gradient of Rayleigh quotient. Briefly, we regarded the functional $\mathfrak{I}_f:g\mapsto \left<g,f\right>$ induced by function $f$ as function $f$ itself.
For all $f\in\mathcal{G}$, we denote $P=\left<f,\Delta f\right>$ and $Q=\left<f,f\right>$, then we have
\begin{equation*}
    \begin{split}
        \frac{\partial \mathfrak{R}(f)}{\partial f}&=\frac{P'Q-PQ'}{Q^2}.
    \end{split}
\end{equation*}
Since 
\begin{equation*}
    P'=\frac{\partial \left<f,\Delta f\right>}{\partial f}=2\Delta f,\quad Q'=\frac{\partial\left<f,f\right>}{\partial f}=2f,
\end{equation*}
and we have
\begin{equation*}
\begin{split}
    \frac{\partial \mathfrak{R}(f)}{\partial f}&=\frac{2\Delta f\cdot\left<f,f\right>-2\left<f,\Delta f\right>\cdot f}{\left<f, f\right>^2}\\
    &=(\Delta - \frac{\left<f,\Delta f\right>}{\left<f,f\right>}I_d)\frac{2f}{\left<f,f\right>}=0
\end{split}
\end{equation*}
where $I_d$ is the identity operator in $\mathcal{G}$, thus $\frac{2f}{\left<f,f\right>}$ is the eigenfunction of $\Delta$, with the corresponding eigenvalue $\frac{\left<f,\Delta f\right>}{\left<f,f\right>}$. Therefore, $f$ is the eigenfunction of $\Delta$.

Since the weighted Laplacian $\Delta$ is equivalent to the weighted Laplacian matrix $L_M$ (see lemma \ref{lem:equiv}), we only need to compute the eigenvectors of weighted Laplacian matrix. Moreover, since $\Delta$ (and $L_M$) is Hermitian, the eigenfunctions (and eigenvectors) are orthogonal to each other, thus the first $k$ smallest eigenfunctions exactly satisfy the orthogonality of indicating functions, as well as solving the relaxation version of weighted cut problem. Finally, in the practical situation, we need to re-convert the real-valued solution to a partition of $G$. For example, the weighted spectral algorithm presented in next section.

\section{Two Theoretical Applications}\label{sec:thy}
We first show the equivalence between relaxed balanced minimum cut problem and relaxed weighted cut problem, as well as the equivalence between weighted cut problem and initial clustering problem that arises in the middle stage of graph partitioning algorithms under a multilevel graph structure. The initial clustering algorithm based on weighted Laplacian method is proposed in later.

\subsection{Equivalence Between Balanced Minimum Cut Problem and Weighted Cut Problem}
As the special cases of weighted cut, existing objective functions like the ratio cut and the normalized cut, in graph partitioning problems have been proposed decades ago and their practical development has been so successful that various application areas rely heavily on them \cite{saxena2017review}. One motivation for constructing these objective functions is that the optimal solution of the traditional minimum cut problem may lead to an unnatural bias especially for partitioning out small sets of points \cite{shi2000normalized}. 
Although a strong and natural relation between the weighted cut and the balanced minimum cut has not been revealed up till now, it is still meaningful in theory to provide a deep understanding for the weighted cut problem in variational language via establishing an equivalence relation from the PDEs point of view.
 
We first give the balance condition for the indicating function, shown as follows,
\begin{equation}\label{eq:bal}
    \int \mathbbm{1}_{C_i}^2 = \int \frac{1}{k}\quad\textrm{for~all~} i=1,2,\ldots,k,
\end{equation}
which means the integral average of $\mathbbm{1}_{C_i}^2$ is $\frac{1}{k}$ for each $i$. Since $\mathbbm{1}_{C_i}^2=\mathbbm{1}_{C_i}$, the partition $\{C_1,C_2,\ldots,C_k\}$ can be regarded as a balanced partition in a sense. By relaxing the discreteness condition, we consider the balanced minimum cut problem defined below
\begin{equation}\label{eq:bal-min-cut}
        \min_{f_i\in\mathcal{G}} \sum_{i=1}^k \left<f_i,\Delta f_i\right>\quad\mathrm{s.t.}~\left<kf_i^2-1,1\right>=0.
\end{equation}
\begin{theorem}
The balanced minimum cut problem is equivalent to the weighted cut problem.
\end{theorem}
We define the functional $\mathfrak{K}$ on $\mathcal{G}^k$ below
\begin{equation*}
    \mathfrak{K}(f_1,f_2,\ldots,f_k)= \sum_{i=1}^k\left<f_i,\Delta f_i\right>, 
\end{equation*}
from the Euler-Lagrange equation with constraints, there exists a constant $\lambda_i$ for each $i$ such that
\begin{equation}\label{eq:el}
    \begin{split}
        \frac{\partial \mathfrak{K}}{\partial f_i}+\frac{\partial}{\partial f_i}\sum_{i=1}^k \lambda_i\left<kf_i^2-1,1\right>=0,
    \end{split}
\end{equation}
since we have
\begin{equation*}
\begin{split}
        \frac{\partial \mathfrak{K}}{\partial f_i}=\frac{\partial \left<f_i,\Delta f_i\right>}{\partial f_i} = 2\Delta f_i,\quad
        \frac{\partial  \left<kf_i^2-1,1\right>}{\partial f_i}=2kf_i,
\end{split}
\end{equation*}
by equation \ref{eq:el}, we know that $\Delta f_i = \lambda_i' f_i$ holds for all $i=1,\ldots,k$ and for some constant $\lambda_i'$. 
Furthermore,the first k smallest eigenfunctions of $\Delta$ solve the balanced minimum cut problem in equation \ref{eq:bal-min-cut} for the reasons that solving the eigenfunctions of $\Delta$ can reduce to computing the eigenvectors of weighted Laplacian matrix $L_M$ and $L_M$ is a Hermitian matrix with its eigenvectors mutually orthogonal, thus exactly satisfying the orthogonality of the indicating functions. We have also presented the weighted Laplacian method in section \ref{sec:prelim}, which is similarly based on solving the eigenfunctions of weighted Laplacian, thus proving that solving the balanced minimum cut problem is equivalent to solving the weighted cut problem.

A remaining question might be why we choose to define the balance condition as $\int \mathbbm{1}_{C_i}^2=\int \frac{1}{k}$ in equation \ref{eq:bal} rather than the simpler form $\int \mathbbm{1}_{C_i}=\int \frac{1}{k}$. The main reason is that the later one would produce Laplace's equation $\Delta f=c$ for some constant $c$, with its trivial solution $f$ be a constant function when $c=0$ \cite{chung2005omega}, and cannot be re-converted to a meaningful partition of a graph. In a word, the existence of the equivalence relation, which has been proved in this subsection via the weighted Laplacian method, can also shed interesting light on the relation between the research of PDEs theroy and recent studies of popular graph problems, thus setting up a bridge to connect these two seemingly unrelated fields.
 
\subsection{Equivalence Between Initial Clustering Problem and Weighted Cut Problem}
The key step in the multilevel graph partitioning framework is to construct an initial partition on a coarsened graph obtained from the first coarsening step. As has been described in the related work, state-of-art graph partition algorithms with multilevel structure are totally heuristic, i.e., they fail to provide any qualitative guarantees for a partition result especially produced by the initial clustering phase. Moreover, although these existing algorithms using the multilevel strategy can indeed largely reduce time complexity of deriving a final partition, comparing with spectral clustering, they always lack in partition accuracy due to their intrinsic heuristic property that appears in the second stage. Additionally, considering that traditional spectral methods with the merit of high accuracy can not be universally transplanted into the initial clustering phase, it is necessary to apply our weighted Laplacian method to the initial clustering phase to make a good compromise between time efficiency and partition accuracy. Meanwhile, it is also worthwhile to mention that the initial clustering algorithm designed with our proposed method possesses robust theoretical support for their partition result due to the inspiration by graph Laplacian theory, and spectral clustering naturally degenerates to a special case of our weighted Laplacian method. In this subsection, we exhibit a rigorous proof about the equivalence between the weighted cut problem and the initial clustering problem, which can set up a bridge to guide us to design an initial clustering algorithm based on weighted Laplacian method.
 
Firstly, we introduce the definition of the initial clustering problem. Suppose $\tilde{G}$ is a coarsened graph of $G$, i.e. the vertex set of $\tilde{G}$ is a set of supernodes $V_{\tilde{G}}=\{ C_i\}_{i=1,\ldots,m}$ where $C_i=\{v_i^t\}_{t=1,\ldots,{j_i}}$ is a supernode. 
Therefore, we know $V_{\tilde{G}}$ forms a partition $\pi=\{ C_i\}_{i=1,\ldots,m}$ of the original graph $G$. For any initial clustering $\tilde{\pi}$ of $\tilde{G}$ in the middle stage of multilevel graph partitioning algorithms, its corresponding original partition of $G$ can be recovered by projecting the cluster of a supernode into the nodes it contained, and this recovered partition can be denoted as $\pi'$. Since $\tilde{\pi}$ is an initial clustering on $\tilde{G}$, we define that $\pi'$ is a coarser partition of $\pi$, i.e. if for all $P\in\pi$, there exists $P'\in\pi'$ such that $P\subset P'$, then we denote it as $\pi'\le\pi$. Thus the \textit{initial clustering problem} is defined as following
\begin{equation*}
    \min_{\pi'\le\pi}\mathrm{Ncut}(\pi').
\end{equation*}
Now we recall the weighted cut problem mentioned in \ref{sec:prelim}. 
 We regarded the coarsened graph $\tilde{G}$ as a doubly-weighted graph defined in \ref{def:dou}.  
 We denote the coarsened graph as $\tilde{G}=(\tilde{V},\tilde{M},\tilde{W})$. The weighted cut problem on $\tilde{G}$ is shown as follows
 \begin{equation*}\textstyle
     \min_{\tilde{\pi}}\mathrm{Wcut}(\tilde{\pi})
 \end{equation*}
 \begin{theorem}
 The weighted cut problem is equivalent to the initial clustering problem.
 \end{theorem}

\subsubsection{Equivalence}
For all partitions $\tilde{\pi}=\{B_1,B_2,\ldots,B_k\}$ on the  coarsened graph $\tilde{G}$ where $B_i=\{C_i^{j_i}\}$, we have
\begin{equation*}
    \begin{split}
        \mathrm{Cut}(B_p,B_q)&=\sum_{t=1}^{j_p}\sum_{s=1}^{j_q}\mathrm{Cut}(C_p^{t},C_q^{s})\\
        \mathrm{mvol}(B_p)&=\sum_{t=1}^{j_p}\mathrm{vol}(C_p^{t}).
    \end{split}
\end{equation*}
We project the partition $\tilde{\pi}$ into the original graph, denoted as $\pi'$, i.e. $\pi'=\{B_i'\}$ where $B_i'=\cup_{t=1}^{j_i}C_i^{t}$, then we can derive that
\begin{equation*}
\begin{split}
        \mathrm{Ncut}(\pi')&=\sum_{p\neq  q}\frac{\mathrm{Cut}(B_p',B_q')}{\mathrm{vol}(B_p')}\\
        &=\sum_{p\neq q}\frac{\sum_{x\in B_p',y\in B_q'}W_{xy}}{\sum_{x\in B_p'}d_x}\\
        &=\sum_{p\neq q}\frac{\sum_{t=1}^{j_p}\sum_{s=1}^{j_q}\sum_{x\in C_p^{t},y\in C_q^{s}}W_{xy}}{\sum_{t=1}^{j_p}\sum_{x\in C_p^{j_t}} d_x}\\
        &=\sum_{p\neq q}\frac{\sum_{t=1}^{j_p}\sum_{s=1}^{j_q}\mathrm{Cut}(C_p^t,C_q^s)}{\sum_{t=1}^{j_p}\mathrm{vol}(C_p^t)}\\
        &=\sum_{p\neq q}\frac{\mathrm{Cut}(B_p,B_q)}{\mathrm{mvol}(B_p)}\\
        &=\mathrm{Wcut}(\tilde{\pi}).
\end{split}
\end{equation*}
 
Hence we have proved that the initial clustering problem is equivalent to the weighted cut problem, which can be solved by the weighted Laplacian method. 

\subsubsection{Weighted Spectral Algorithm}\label{sec:algo}

We describe our proposed weighted spectral algorithm based on weighted Laplacian method via the above equivalence relation. 

Algorithm \ref{algo:coarsening} takes as its inputs the original graph $G$, a partition $\pi$ and the number of clusters $k$, then outputs the $k$-clustering result of initial clustering. 
Algorithm \ref{algo:coarsening} gives the pseudo-code of our algorithm. Line 1-8 performs the summation operation on rows and columns of the matrix, which can be executed very efficiently in a parallel system.
 
Another part of algorithm \ref{algo:coarsening} (corresponds the line 9-14 in the pseudo-code)  produces an initial clustering of the coarsened graph with $m$ vertices, based on weighted Laplacian methods. Like traditional spectral clustering, we compute the first $k$ smallest eigenvectors of the weighted Laplacian matrix $L_M$ to obtain a matrix $U\in\mathbb{R}^{m\times k}$ which takes these eigenvectors as its columns, then run $k$-means algorithm on the rows of $U$ to obtain a $k$-clustering result.

\begin{algorithm}[htbp]
\caption{Weighted Spectral Algorithm for Initial Clustering Based on Weighted Laplacian Method}
\label{algo:coarsening}
\begin{algorithmic}[1]
\REQUIRE $G, \pi, k$
\ENSURE an initial clustering result $\{C_i\}_{i=1,\ldots,k}$
\STATE $(V,M,W)\gets G$
\STATE Compute the degree matrix $D$ of $W$
\FOR {$c_i$ is a cluster in the partition $\pi$}
\STATE $\{v_i^1<v_i^2<\ldots<v_i^{j_i}\}\gets c_i$
\STATE $\tilde{W}(:,i)\gets \sum_{k=1}^{j_i}W(:,v_i^{k})$
\STATE $\tilde{W}(i,:)\gets \sum_{k=1}^{j_i}W(v_i^{k},:)$
\STATE $\tilde{M}(i,i)\gets \sum_{k=1}^{j_i}D(v_i^{k},v_i^{k})$
\ENDFOR
\STATE Compute the degree matrix $\tilde{D}$ of $\tilde{W}$
\STATE $L_M\gets \tilde{M}^{-1/2}(\tilde{D}-\tilde{W})\tilde{M}^{-1/2}$
\STATE $m\gets\mathrm{dim}(\tilde{W})$
\STATE Conduct the eigendecomposition $L_Mx=\lambda x$, the eigenvectors corresponding to the $k$ smallest eigenvalues are $\{x_i\}_{i=1,\ldots,k}$, let $U:=(x_i)\in\mathbb{R}^{m\times k}$.
\STATE Execute the $k$-means algorithm on the rows of $U=(u_j)$ to obtain a clustering $B_i$
\STATE For all vertex, $v\in C_i$ if $u_j$ contains $v$ and $u_{j}\in B_i$
\RETURN $\{C_i\}_{i=1,\ldots,k}$
\end{algorithmic}
\end{algorithm}

\section{Experiments}\label{sec:expr}

\begin{figure*}[htbp]
\centering
\subfigure[8 clusters]{\includegraphics[width=0.32\textwidth]{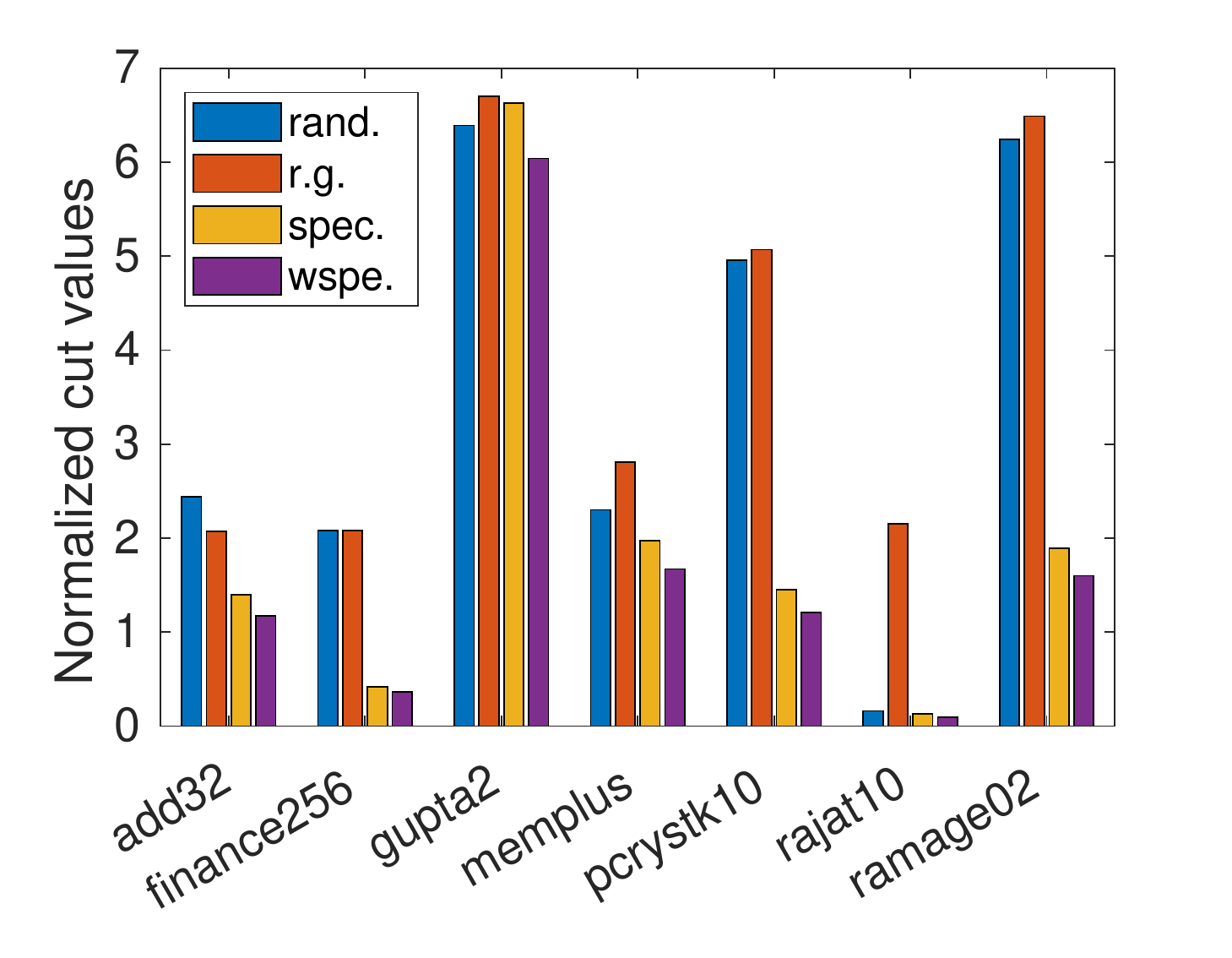}}
\subfigure[16 clusters]{\includegraphics[width=0.32\textwidth]{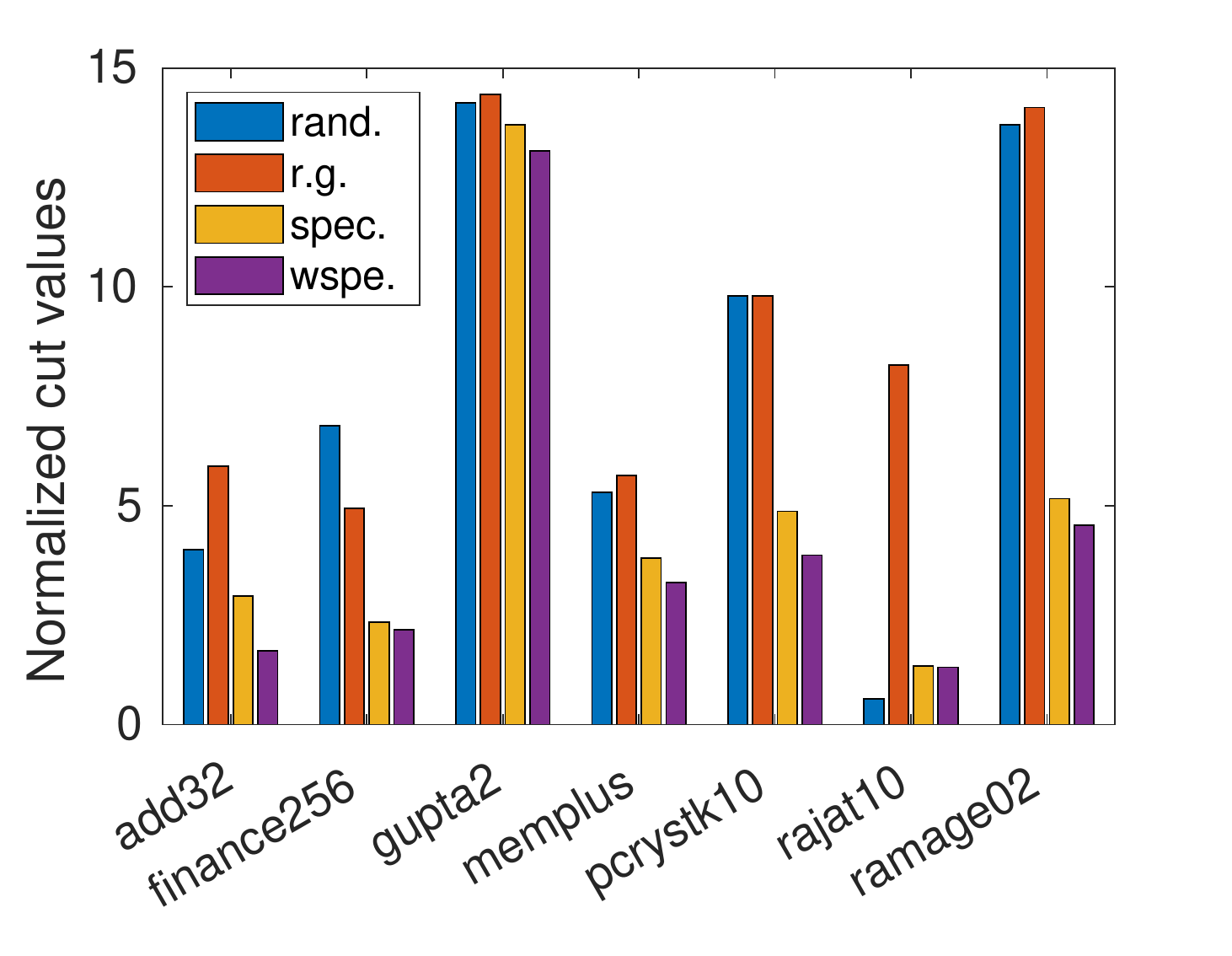}}
\subfigure[32 clusters]{\includegraphics[width=0.32\textwidth]{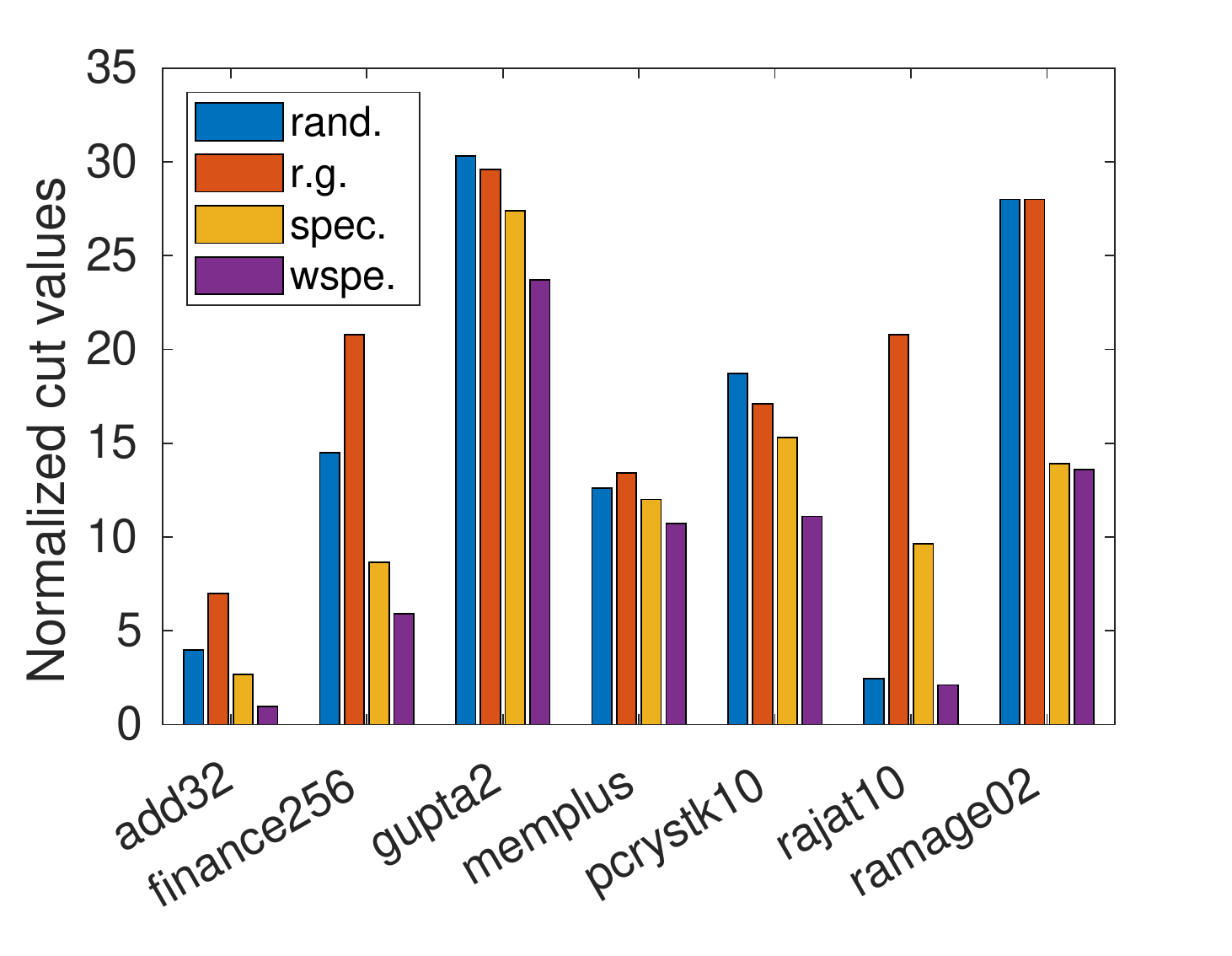}}
\caption{Clustering results on the real-world datasets (lower cut values are better). Our weighted spectral (wspe.) outperforms existing algorithms. Each result in the bar charts is the average of running for 10 times.}
\label{fig:dataset}
\end{figure*}

\begin{table*}[htbp]
    \centering\caption{Normalized cut values with different initial clustering algorithms (lower cut values are better). Our weighted spectral algorithm always performs the best result on different scales of cluster tasks. Each result is the average for running $10$ times.}
    \begin{tabular}{|l|r|r|r|r|r|r|r|r|}\hline
        \# clusters          &4      &8      &16     &32     &64     &128    &256    &512    \\\hline
         random (rand.)             &.393   &.529   &2.39   &3.96   &15.7   &37.3   &94.3   &251    \\\hline
         region-growing (r.g.)     &.162   &1.44   &3.33   &7.00   &18.1   &34.0   &85.0   &186    \\\hline
         spectral (spec.)           &.345   &.269   &1.70   &4.81   &10.5   &30.0   &80.5   &181    \\\hline
         \textbf{weighted spectral (wspe.)}  &.158   &.214   &1.69   &2.34   &8.70   &22.6   &70.0   &177    \\\hline
    \end{tabular}
    \label{tab:values}
\end{table*}

\begin{table}[htbp]
    \centering\caption{Sizes of test graphs}
    \begin{tabular}{|c|l|r|r|}\hline
      No. & Graph name          &\# nodes      &\# edges           \\\hline
        1 & add32             &4960   &9462   \\\hline
        2& finance256     &37376   &130560     \\\hline
        3& gupta2           &62064   &2093111   \\\hline
        4& memplus &17758 &54196   \\\hline
        5& pcrystk02&13965&477309\\\hline
        6& rajat10&30202&50101\\\hline
        7& ramage02&16830&1424761\\\hline
    \end{tabular}
    \label{tab:graphs}
\end{table}

In this section, we present some experimental results to show the effectiveness and robustness of our weighted spectral algorithm, which outperforms the other initial clustering algorithms in terms of normalized cut values. Our experiment framework is modified based on KaHIP \cite{sanders2013think}, which is an open-source framework written in C++ for multilevel graph partitioning. Moreover, the region-growing algorithm is utilized as the initial clustering algorithm in the framework.

Recall the initial clustering algorithms listed in section \ref{sec:related}, our compared experimental initial clustering algorithms contain region-growing (r.g.), which performs well than recursive bisection in practice, based on some previous reports \cite{predari2017comparison}. Besides, we also choose spectral clustering (spec.) as initial clustering algorithm to compared with our weighted spectral (wspe.), thus to show the progress of our weighted Laplacian method. As an extremely case, we choose the random initial clustering (rand.) as an experiment, in which every supernodes are labeled with the random labels from $1$ to $k$ where $k$ is the number of clusters. All experiments were compiled with complied with g++ 7.4.0, and performed for 10 times on a Linux machine with Intel (R) Xeon(R) Bronze 3104 CPU@1.70 GHz$\times$12 and 32 Gbyte of main memory, to obtain the average results.

We compare our algorithm to the existing algorithms. We run the partitioning tasks with different scale from 8 to 512 on the dataset add32, and the results are contained in table \ref{tab:values}, finally weighted spectral always returns the best results. Moreover, we run the partitioning tasks on different real-world datasets (see table \ref{tab:graphs}). Each result is the average of running the same experiment 10 times, and the results are plotted on the bar charts in figure \ref{fig:dataset}.

\section{Conclusion}\label{sec:conc}
We studied the weighted cut problem and showed it can be solved based on our proposed weighted Laplacian method, which is a powerful tool for analyzing the doubly-weighted graph. We show that the relaxed balanced minimum cut problem is equivalent to the relaxed weighted cut problem, as well as the initial clustering problem that arises in the middle stage in the graph partitioning algorithm with multilevel structure, thus all of these can be solved by our weighted Laplacian method. Our experimental algorithm based on weighted Laplacian method outperforms the existing techniques in terms of the normalized cut values. 

We have proved that the relaxed weighted cut problem is equivalent to the relaxed balanced minimum cut problem, and our work shows the potential applications of PDEs methods in the variants of the weighted cut problem (e.g. the normalized cut problem and the ratio cut problem).  In the future, we would like to attempt to utilize the PDEs methods such as the finite element method and Monte Carlo method to handle some traditional graph partitioning problems in the graph analysis.

\section{Acknowledgements}
The work of this article is supported by Cyber Security Project of Ministry of Science and Technology of China, Research on Key Technologies of Internet of Things and Security Guarantee of Smart City (2018YFB080340), National Natural Science Fund for distinguished young scholars (No. 61625205), Key Research Projects in Frontier Science of the Chinese Academy of Sciences (QYZDY-SSW-JSC002).


\end{document}